\documentclass[nohyperref]{article}

\usepackage{microtype}
\usepackage{graphicx}
\usepackage{subfigure}
\usepackage{booktabs} 

\usepackage{tabularx}
\usepackage{array}
\newcolumntype{x}[1]{>{\centering\let\newline\\\arraybackslash\hspace{0pt}}p{#1}}
\usepackage{multirow}

\usepackage{hyperref}

\usepackage{stackengine}
\newcommand\barbelow[1]{\stackunder[1.2pt]{$#1$}{\rule{.8ex}{.075ex}}}

\usepackage[accepted]{icml2022}

\usepackage{amsmath}
\usepackage{amssymb}
\usepackage{mathtools}
\usepackage{amsthm}

\usepackage[capitalize,noabbrev]{cleveref}

\theoremstyle{plain}
\newtheorem{theorem}{Theorem}[section]

\theoremstyle{definition}
\newtheorem{definition}[theorem]{Definition}

\theoremstyle{remark}

\usepackage[textsize=tiny]{todonotes}

\def\our{InterContiNet}
\def\ourtraining{Hyperrectangle Training}

\DeclareMathOperator*{\argmax}{arg\,max}
\DeclareMathOperator*{\argmin}{arg\,min}

\def\R{\mathbb{R}}
\def\N{\mathbb{N}}
\def\P{\mathbb{P}}

\def\e{\varepsilon}

\usepackage{pifont}
\icmltitlerunning{Continual Learning with Guarantees via Weight Interval Constraints}

\begin{document}

\twocolumn[

\icmltitle{Continual Learning with Guarantees via Weight Interval Constraints}



\icmlsetsymbol{equal}{*}

\begin{icmlauthorlist}

\icmlauthor{Maciej Wo\l{}czyk}{equal,uj}
\icmlauthor{Karol J. Piczak}{equal,uj}
\icmlauthor{Bartosz W\'ojcik}{uj}
\icmlauthor{\L{}ukasz Pustelnik}{uj}
\icmlauthor{Pawe\l{} Morawiecki}{ipipan}
\icmlauthor{Jacek Tabor}{uj}
\icmlauthor{Tomasz Trzci\'nski}{uj,pw,tpx}
\icmlauthor{Przemys\l{}aw Spurek}{uj}
\end{icmlauthorlist}

\icmlaffiliation{uj}{Faculty of Mathematics and Computer Science, Jagiellonian University, Kraków, Poland}
\icmlaffiliation{pw}{Warsaw University of Technology, Warsaw, Poland}
\icmlaffiliation{tpx}{Tooploox}
\icmlaffiliation{ipipan}{Institute of Computer Science, Polish Academy of Sciences, Warsaw, Poland}

\icmlcorrespondingauthor{Maciej Wołczyk}{maciej.wolczyk@doctoral.uj.edu.pl}

\icmlkeywords{Machine Learning, ICML}

\vskip 0.3in
]



\printAffiliationsAndNotice{\icmlEqualContribution} 

\begin{abstract}
    We introduce a new training paradigm that enforces interval constraints on neural network parameter space to control forgetting. 
    Contemporary Continual Learning (CL) methods focus on training neural networks efficiently from a stream of data, while reducing the negative impact of catastrophic forgetting, yet they do not provide any firm guarantees that network performance will not deteriorate uncontrollably over time. 
    In this work, we show how to put bounds on forgetting by reformulating continual learning of a model as a continual contraction of its parameter space. To that end, we propose~\ourtraining{}, a new training methodology where each task is represented by a hyperrectangle in the parameter space, fully contained in the hyperrectangles of the previous tasks. This formulation reduces the NP-hard CL problem back to polynomial time while providing full resilience against forgetting. 
    We validate our claim by developing \our{} (Interval Continual Learning) algorithm which leverages interval arithmetic to effectively model parameter regions as hyperrectangles. 
    Through experimental results, we show that our approach performs well in a continual learning setup without storing data from previous tasks. 
\end{abstract}
\section{Introduction}

Learning from a continuous stream of data is natural for humans, as new experiences come sequentially in our life. Yet, artificial neural network models fail to exhibit the very same skill~\cite{mccloskey1989catastrophic,ratcliff1990connectionist,french1999catastrophic,DBLP:journals/corr/GoodfellowMDCB13}. Although they generally deal well with solving increasingly complex tasks, their inability to acquire new knowledge without {\it catastrophically forgetting} what they learned previously is considered one of the critical roadblocks to reaching human-like intelligence.

\begin{figure}[t!]
\begin{center}
\subfigure[General CL setting.]
{
\label{fig:imp_dec_1_a}
  \scalebox{0.30}{
    \begin{tikzpicture}
\draw[fill=yellow] (-1,1) ellipse (1 and 2);

\node at (-1.5,0.0) { \large $W_1$};
\node at (-1.2,-0.3) { \large $\times$};

\node at (-1,4.3) { \LARGE $TASK \ 1$};


\begin{scope}
\clip (3.1,0.2) rectangle (4.2,2.8);
\clip (3,1) ellipse (1 and 2);
\fill[color=yellow] (3,1) ellipse (1 and 2);
\end{scope}

\draw (3,1) ellipse (1 and 2);

\draw[]  (3.1,0.2) rectangle (4.2,2.8);

\node at (2.5,0.0) { \large $W_1$};
\node at (2.8,-0.3) { \large $\times$};

\node at (3.15,2.5) { \large $W_2$};
\node at (3.45,2.2) { \large $\times$};

\draw [>=stealth] (2.8,-0.3) -- (3.45,2.2);

\node at (3,4.3) { \LARGE $TASK \ 2$};


\begin{scope}
\clip (7.1,0.2) rectangle (8.2,2.8);
\clip (7,1) ellipse (1 and 2);
\clip (9,0.3) -- (6.5,-1.5) -- (6.5,2.1) -- cycle;
\fill[color=yellow]  (7,1) ellipse (1 and 2);
\end{scope}

\node at (7.15,2.5) { \large $W_2$};
\node at (7.45,2.2) { \large $\times$};

\node at (7.15,1.0) { \large $W_3$};
\node at (7.45,0.7) { \large $\times$};

\draw [>=stealth] (7.45,2.2) --  (7.45,0.7);

\node at (7,4.3) { \LARGE $TASK \ 3$};

\draw (7,1) ellipse (1 and 2);

\draw  (7.1,0.2) rectangle (8.2,2.8);

\draw(9,0.3) -- (6.5,-1.5) -- (6.5,2.1) -- cycle;

\end{tikzpicture}
  } 
}
\subfigure[\our{} CL setting.]
{
\label{fig:imp_dec_1_b}
  \scalebox{0.30}{
    \begin{tikzpicture}

\draw[fill=yellow] (-5,1) ellipse (1 and 2);

\node at (-5.5,0.0) { \large $W_1$};
\node at (-5.2,-0.3) { \large $\times$};

\node at (-3,4.3) { \LARGE $TASK \ 1$};
\draw [dashed] (1,-2) -- (1,5);

\draw (-1,1) ellipse (1 and 2);
\draw[fill=orange]  (-1.8,-0.1) rectangle (-0.2,2.1);

\node at (-1.5,0.0) { \large $W_1$};
\node at (-1.2,-0.3) { \large $\times$};

\node at (-1.3,1.3) { \large $W_2$};
\node at (-1,1) { \large $\times$};

\draw [>=stealth] (-1.2,0) --  (-1,1);



\begin{scope}
\clip (3.8,-0.1) rectangle (2.2,2.1);
\clip (3.1,0.2) rectangle (4.2,2.8);
\fill[color=yellow] (3.8,-0.1) rectangle (2.2,2.1);
\end{scope}

\draw (3,1) ellipse (1 and 2);
\draw[]  (3.8,-0.1) rectangle (2.2,2.1);

\draw[]  (3.1,0.2) rectangle (4.2,2.8);

\node at (2.7,1.3) { \large $W_2$};
\node at (3,1) { \large $\times$};

\node at (3.15,1.5) { \large $W_3$};
\node at (3.45,1.2) { \large $\times$};

\draw [>=stealth] (3,1) --  (3.45,1.2);

\node at (5,4.3) { \LARGE $TASK \ 2$};
\draw [dashed] (9,-2) -- (9,5);


\begin{scope}
\clip (7.8,-0.1) rectangle (6.2,2.1);
\clip (7.1,0.2) rectangle (8.2,2.8);
\fill[color=orange] (7.8,-0.1) rectangle (6.2,2.1);
\end{scope}

\draw (7,1) ellipse (1 and 2);
\draw[]  (7.8,-0.1) rectangle (6.2,2.1);

\draw[]  (7.1,0.2) rectangle (8.2,2.8);

\node at (6.7,1.3) { \large $W_2$};
\node at (7,1) { \large $\times$};

\node at (7.15,1.5) { \large $W_3$};
\node at (7.45,1.2) { \large $\times$};

\draw [>=stealth] (7,1) --  (7.45,1.2);



\begin{scope}
\clip (11.8,-0.1) rectangle (10.2,2.1);
\clip (11.1,0.2) rectangle (12.2,2.8);
\fill[color=yellow] (13,0.3) -- (10.5,-1.5) -- (10.5,2.1) -- cycle;
\end{scope}

\draw (11,1) ellipse (1 and 2);
\draw  (11.8,-0.1) rectangle (10.2,2.1);

\draw  (11.1,0.2) rectangle (12.2,2.8);

\draw (13,0.3) -- (10.5,-1.5) -- (10.5,2.1) -- cycle;

\node at (13,4.3) { \LARGE $TASK \ 3$};

\node at (11.15,1.5) { \large $W_3$};
\node at (11.45,1.2) { \large $\times$};


\begin{scope}
\clip (15.8,-0.1) rectangle (14.2,2.1);
\clip (14.5, -0.4) rectangle (16.0, 1.05);
\clip (15.1,0.2) rectangle (16.2,2.8);
\fill[color=orange] (17,0.3) -- (14.5,-1.5) -- (14.5,2.1) -- cycle;
\end{scope}

\draw (15,1) ellipse (1 and 2);
\draw (15.8,-0.1) rectangle (14.2,2.1);

\draw (15.1,0.2) rectangle (16.2,2.8);

\draw (17,0.3) -- (14.5,-1.5) -- (14.5,2.1) -- cycle;
\draw (14.5, -0.4) rectangle (16.0, 1.05);


\node at (15.15,1.5) { \large $W_3$};
\node at (15.45,1.2) { \large $\times$};

\node at (15.15,0.45) { \large $W_4$};
\node at (15.45,0.75) { \large $\times$};

\draw [>=stealth] (15.45,1.2) --  (15.45,0.75);

\end{tikzpicture}
  } 
}
\end{center}
\vspace*{-1.0em}
\caption{
Visualization of parameter space regions that perform well on tasks learned sequentially. For a new task, weights $\theta_i$ must lie at the intersection of consecutive regions for the model to perform well on previous tasks. In the general CL setting, shown in Figure~\ref{fig:imp_dec_1_a} from~\cite{DBLP:conf/icml/KnoblauchHD20}, such areas can have arbitrary shapes. We introduce the \ourtraining{} paradigm where the parameter spaces are modeled with hyperrectangles, colored orange in Figure~\ref{fig:imp_dec_1_b}. This formulation puts guarantees on forgetting while reducing the calculation of NP-hard CL problem back to polynomial time.}
\vspace{-0.5em}
\label{fig:teaser}
\end{figure}
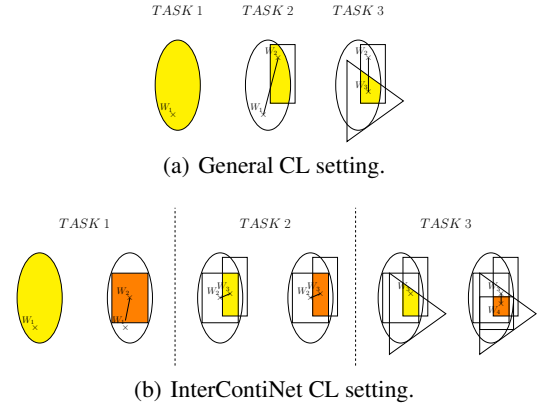

Continual learning is a rapidly growing field of machine learning that aims to solve this limitation and to bridge the gap between human and machine intelligence. 
Although several methods effective at reducing forgetting when learning new tasks were proposed~\cite{kirkpatrick2017overcoming,DBLP:conf/nips/LeeKJHZ17,li2017learning,LopezPaz2017GradientEM,Shin2017ContinualLW,zenke2017continual,aljundi2018memory,Masse2018AlleviatingCF,Liu2018RotateYN,Rolnick2019ExperienceRF,vandeVen2020BraininspiredRF}, they typically do not provide any firm guarantees about the degree of forgetting experienced by the model, which renders them inapt for safety-critical applications, such as autonomous driving or robotic manipulation. 
For instance, an autonomous car that was trained online during winter cannot forget how to drive in the snow, after acquiring new driving skills during summer. 
On the other hand, methods that do guarantee lack of forgetting, e.g.~\citet{rusu2016progressive,mallya2018packnet,Mallya2018PiggybackAA}, operate in conditions that are far from real life --- they assume that the task identity is known at inference, and hence we are allowed to adapt the network for each task. 

In this work, we identify those limitations of prior research and introduce a new learning paradigm for continual learning which puts strict bounds on forgetting, while being suitable for a more realistic single-head scenario. More specifically, we formulate continual learning as a constrained optimization problem, where when learning a new task $T$ the optimization needs to remain in the  parameter region that prevents performance drop on all previous tasks $1, \ldots, T-1$. In general, this problem is NP-hard \cite{DBLP:conf/icml/KnoblauchHD20} because the shapes of the viable parameter regions are highly irregular and overlaps are hard to identify. To overcome this problem, we propose to look for the intersections of hyperrectangles which are subsets of these regions, as displayed in Figure~\ref{fig:teaser}. We show that finding those intersections can be done in polynomial time, as long as we trim the parameter regions to fit the consecutive hyperrectangles. The resulting training paradigm, dubbed \ourtraining{}, is able to learn from a sequence of tasks while putting bounds on catastrophic forgetting. 

Although this simple idea is theoretically viable, finding hyperrectangles that follow these requirements and performing optimization with constrained parameter space is not trivial with the existing neural network models.
To solve this problem, we propose \our{}, a novel neural algorithm 
based on interval arithmetic which we use to implement the \ourtraining{}. Instead of optimizing individual network parameters, we consider intervals of parameter values, ensuring that every parameter value within the interval does not lead to performance deterioration with respect to the previous tasks.
More specifically, once training on task $T$ is finished, we use the obtained intervals as the hyperrectangle constraints for training the following tasks $(T+1, T+2, \ldots)$.
We show how \our{} can be incorporated within conventional neural network building blocks such as linear layers, convolutional layers, and ReLU activations. The resulting combination of \ourtraining{} and \our{} algorithm offers competitive results on a range of continual learning tasks we evaluated.

To summarize, our main contributions are the following:
\begin{itemize}
    \setlength\itemsep{0.1em}
    \item A novel paradigm of continual learning called \ourtraining{} that models optimal parameter space with hyperrectangles, allowing us to control catastrophic forgetting in neural networks and set guaranteed performance.
    We show that with adequate parameterization we can use this paradigm to train neural networks in polynomial time. 
    \item \our{}, a new algorithm that implements \ourtraining{} in practice, leveraging interval arithmetic of model parameters.
    \item Thorough experimental evaluation that confirms the validity of our method and its potential for putting guarantees on catastrophic forgetting in contemporary deep networks. 
\end{itemize}

\section{Related Works}

\paragraph{Continual learning}

Traditionally, continual learning approaches are grouped into three general families: regularization, dynamic architectures, and replay~\cite{parisi2019continual,delange2021continual}.

Regularization-based methods extend the loss function with additional terms slowing down changes in model parameters crucial for previous tasks. These terms can be introduced individually for each task or as a single overall value representative for the whole sequence (\textit{online} variant). Based on the type of regularization, a further subdivision splits these methods into prior- and data-focused approaches.

Prior-focused methods employ a prior on the model parameters when learning new tasks. The importance of individual parameters is estimated, i.a., with the Fisher information matrix, as shown in the seminal work on Elastic Weight Consolidation (EWC)~\cite{kirkpatrick2017overcoming}. A~similar approach, Synaptic Intelligence~\cite{zenke2017continual}, utilizes the whole learning trajectory to compute this importance measure. In contrast, the Memory Aware Synapses technique (MAS)~\cite{aljundi2018memory} approximates the importance of the parameters based on the gradients of the squared $L_2$-norm of the learned function output. Our proposed method bears a certain resemblance to these methods, as it also regularizes the model's parameters. However, we approach this problem differently by abandoning the soft quadratic penalties and instead introducing hard constraints based on the interval propagation loss. This allows us to provide guarantees on the worst-case level of forgetting.

Data-focused methods instead perform knowledge distillation from models trained on previous tasks when learning on new data. For instance, the Learning without Forgetting paradigm (LwF)~\cite{li2017learning} employs an additional distillation loss based on the comparison of new task outputs generated by the new and old models.

Regularization-based methods have a number of advantages. They work without changing the structure of the model and without storing any examples of the old data, which might be crucial due to data privacy issues. However, as they only impose a soft penalty, both the prior- and data-focused regularization methods cannot entirely prevent the forgetting of previous knowledge.

Such guarantees can be provided instead by approaches with dynamic architectures, which dedicate separate model branches to different tasks. These branches can be grown progressively, such as in the case of Progressive Neural Networks~\cite{rusu2016progressive}. Alternatively, a static architecture can be reused with iterative pruning proposed by PackNet~\cite{mallya2018packnet} or by using Supermasks in Superposition~\cite{Wortsman2020SupermasksIS}. A major practical drawback of these methods is that they require the knowledge of the actual task identity during inference, which is problematic in more realistic scenarios. Therefore, due to this limitation, we do not consider them further in our analysis.

A different, very successful approach in continual learning is replaying some form of old data during incremental training to maintain the knowledge acquired in the past~\cite{LopezPaz2017GradientEM,Shin2017ContinualLW,Rolnick2019ExperienceRF}. This process can use actual raw examples from the previous tasks or samples generated synthetically by a separate model. However, this paradigm is inapplicable in various applications, such as working with private medical data. Thus, in further experiments, we limit our comparisons to approaches based on regularization.

Finally, a recent work~\cite{DBLP:conf/iclr/MirzadehFGP021} addressed catastrophic forgetting by exploiting the linear connectivity between solutions obtained through multitask and continual learning in the form of Mode Connectivity SGD. Our work shares its fundamental motivation, i.e. driving the optimization process to parameter subspaces appropriate for the whole sequence of tasks. However, in contrast to~\citet{DBLP:conf/iclr/MirzadehFGP021}, we do not employ any kind of replay buffer, but we establish the solution boundaries based on interval arithmetic.

\paragraph{Interval arithmetic for neural networks}
In deep learning, interval arithmetic was used in three main areas. First of all, we can use interval arithmetic to deal with the situation where we have only uncertain information about the data.  
\citet{chakraverty2014interval} presented an interval artificial neural network (IANN) that can handle input and output data represented as intervals. This architecture was further adapted for viruses data with uncertainty modeled by intervals~\cite{chakraverty2017novel,sahoo2020structural}.

In~\citet{gowal2018effectiveness,morawiecki2019fast} interval arithmetic was used to produce neural networks that are robust against adversarial attacks. In such an approach authors used worst-case cross-entropy to train the classification model.
In~\citet{proszewska2021hypercube}, the architecture was employed for representing voxels in 3D shape modeling.

In contrast to these previous works, we apply interval arithmetic to weights instead of input data.

\section{ \ourtraining{}} \label{sec:theoretical}

In this section, we will present \ourtraining{} for the continual learning problem. The main idea is to constrain the parameter search within the set of parameters for which any particular solution is valid for the previous tasks. In turn, we can guarantee that any parameter obtained this way will still be a valid solution for the previous task, thus putting bounds on forgetting.
Although explicitly representing the region of valid solutions is not tractable, we show that using interval arithmetic we can efficiently compute the upper bound of the worst-case loss in this region to facilitate practical implementations. This general approach is presented in Figure~\ref{fig:teaser}.

\paragraph{The continual learning problem}
Assuming a sequence of tasks $T_i$ for $i  = 1, \ldots, k$ we formulate the continual learning objective for task $T_j$ following \citet{DBLP:conf/iclr/ChaudhryRRE19} as:
\begin{equation}
    \argmin_\theta \ell(T_j, \theta) \text{ satisfying } \ell(T_m, \theta) \leq \ell(T_m, \theta_m^*),
    \label{eq:cl_equation}
\end{equation} 
for all $ m  = 1, \ldots, j - 1$ , where $\theta_m^*$ are the parameters obtained directly after learning task $m$ and $\ell$ is the cross-entropy loss over the whole task dataset $\ell(T_j, \theta) = \frac{1}{|T_j|} \sum_{(x, y) \in T_j} l(\phi_\theta(x), y)$, where by $\phi_\theta(x)$ we denote the (pre-softmax) logits produced by the neural network $\phi$ with parameters $\theta$.
In other words, we would like to train our model on $j$-th task without reducing the performance of previous tasks $m = 1, \ldots, j-1$. 

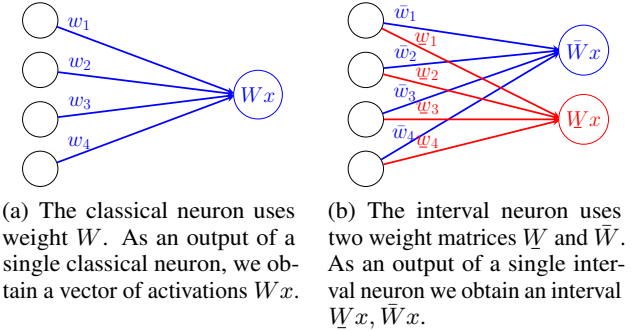
\begin{figure}[t]
\begin{center}
\subfigure[The classical neuron uses weight $W$. As an output of a single classical neuron, we obtain a vector of activations $Wx$. ]
{
\label{fig:neurons_a}
  \scalebox{0.47}{
    \tikzset{%
  input neuron/.style={
    circle,
    draw,
    minimum size=1cm
  },
  neuron missing/.style={
    draw=none, 
    scale=4,
    text height=0.333cm,
    opacity=.0,
    execute at begin node=\color{black}$\vdots$
  },
}

\begin{tikzpicture}[x=1.4cm, y=1.4cm, >=stealth]

\foreach \m/\l [count=\y] in {1,...,4}
  \node [input neuron/.try, neuron \m/.try] (input-\m) at (0,1.0-\y) {};

\foreach \m [count=\y] in {1}
  \node [every neuron/.try, neuron \m/.try ] (hidden1-\m) at (4,-0.5-\y) {};


\foreach \i in {1,...,4}
  \foreach \j in {1}
    \draw [->, blue, line width=0.5mm] (input-\i) -- (hidden1-\j) node [above, very near start] {\LARGE $w_\i$};

  \foreach \m [count=\y] in {1}
  \node [draw, circle, blue, align=center, minimum size=1cm] (hidden1-\m) at (4.4,-0.5-\y) { \LARGE $ W x $};

\end{tikzpicture}
  } 
}
\quad
\subfigure[The interval neuron uses two weight matrices $\barbelow W$ and $\bar W$. As an output of a single interval neuron we obtain an interval $\barbelow W x,  \bar W  x$. ]
{
\label{fig:neurons_b}
  \scalebox{0.47}{
    \tikzset{%
  input neuron/.style={
    circle,
    draw,
    minimum size=1cm,
  },
}

\begin{tikzpicture}[x=1.4cm, y=1.4cm, >=stealth]

\foreach \m/\l [count=\y] in {1,2,3,4}
  \node [input neuron/.try, neuron \m/.try] (input-\m) at (0,0.8-\y) {};

\foreach \m [count=\y] in {1}
  \node [every neuron/.try, neuron \m/.try ] (hidden1-\m) at (4,0.2-\y) {};

\foreach \i in {1,...,4}
  \foreach \j in {1}
    \draw [->,blue, line width=0.5mm] (input-\i) -- (hidden1-\j) node [above, very near start] { \LARGE $\bar w_\i$};
    

\foreach \m [count=\y] in {1}
  \node [every neuron/.try, neuron \m/.try ] (hidden2-\m) at (4,-1.2-\y) {}; 
  

\foreach \i in {1,...,4}
  \foreach \j in {1}
    \draw [->,red, line width=0.5mm] (input-\i) -- (hidden2-\j) node [above,  near start] {\LARGE $\barbelow w_\i$};
    
  
  \foreach \m [count=\y] in {1}
  \node [draw, circle, blue, align=center, minimum size=1cm] (hidden1-\m) at (4.4,0.2-\y) {\LARGE $ \bar W x $};
  \foreach \m [count=\y] in {1}
  \node [draw, circle, red, align=center, minimum size=1cm] (hidden1-\m) at (4.4,-1.2-\y) {\LARGE $ \barbelow W x $};

\end{tikzpicture}
  }
}
\end{center}

\caption{ Comparison between the classical neuron and the interval neuron. In the interval neuron, we have two sets of weights describing the intervals.}
\label{fig:neurons}
\end{figure}
    
Since this constraint is difficult to satisfy given the limited access to data from previous tasks, most continual learning methods transform it into a soft penalty.  In particular, regularization-based continual learning methods such as EWC \cite{kirkpatrick2017overcoming} usually use a soft approximation of the forgetting constraint. While this approach reduces changes in weights which are considered important for the previous task, it does not completely eliminate forgetting. Rephrasing this in the context of stability-plasticity trade-off, most methods allow for better plasticity at the cost of sometimes sacrificing stability. 
However, in many applications, such as medical imaging, autonomous vehicles, or robotics, even slight hints of forgetting are not acceptable. Our main goal is to create a CL method that does not allow any forgetting while maintaining good plasticity.

\paragraph{ \ourtraining{}{} }
In order to fully realize the constraint in Eq. \eqref{eq:cl_equation} we change the approach of thinking about the parameter space of the neural networks, switching the focus from finding particular points $\theta \in \mathbb{R}^D$ in the parameter space to reasoning about whole regions $\Theta \subset \mathbb{R}^D$. In other words, in the continual learning problem of optimizing over $T_1, \ldots T_n$, when training on task $T_j$, we are interested in finding:
\begin{equation}
    \begin{split}
        \Theta_{j + 1} = &\argmin_\Theta \max_{\theta \in \Theta_j} \ell(T_{j + 1}, \theta) \\
        &\text{ satisfying } \|\Theta\| > C_{j+1}, \Theta \subseteq \Theta_j
        \label{eq:region_optimization_continual}
    \end{split}
\end{equation} 
where $\|\Theta\|$ denotes some notion of the size of the set, e.g. its volume, $C_{j+1}$ is a hyperparameter denoting the minimal required size of that set, and we set $\Theta_0 = \mathbb{R}^D$ to be the whole domain of real vectors.
With this formulation, since $\Theta_{j+1} \subseteq \Theta_{j}$ we can easily see that for any $\theta \in \Theta_{j+1}, \ell(T_j, \theta) \leq \max_{\theta' \in \Theta_j} \ell(T_j, \theta')$. Thus, this approach allows us to provide guarantees about the level of forgetting.

However, applying this idea directly to neural networks is intractable. In fact, as \citet{DBLP:conf/icml/KnoblauchHD20} noted, even finding the intersection between regions $\Theta_j$ and $\Theta_{j+1}$ of an arbitrary shape is an NP-hard problem. In order to bring the problem back to tractable regime, we try to instead find  $\bar{\Theta} \subseteq \Theta$, where $\bar{\Theta}$ has a restricted shape that can be easily intersected. That is, we want to find:
\begin{equation}
    \begin{split}
        \bar{\Theta}_{j} &= \argmax_{\bar{\Theta}} \| \bar{\Theta} \| \text{ satisfying } \bar{\Theta} \subseteq \Theta, \bar{\Theta} \in \mathcal{G},
    \end{split}
\end{equation}
where $\mathcal{G}$ is a family of sets with required properties. For example, in order to solve the problem of NP-hard intersections we could use $\mathcal{G}$ as the set of convex polytopes. However, this restriction is not enough to make optimization of Eq.  \eqref{eq:region_optimization_continual} tractable. Instead, we restrict this family even further, and we set $\mathcal{G}$ to be a set of hyperrectangles, which, as we will show, allows us to develop a viable algorithm for training neural networks in the continual learning setting. With this formulation, we lose some plasticity as we are not able to represent points in the region $\Theta_j \setminus \bar{\Theta}_j,$ but in the experimental section we show that the final model still has enough expressiveness to solve difficult tasks.

Using this formulation, we will now show how to apply the \ourtraining{} to continual learning problem with neural networks. In particular, we will demonstrate that for standard modern neural network architectures we can find a tractable and easily differentiable upper bound of $\max_{\theta \in \Theta} \ell(T_j, \theta)$ which will in turn allow us to perform optimization of Eq. \eqref{eq:region_optimization_continual}.

\paragraph{Interval arithmetic for neural networks}

Let us consider the family of feed-forward neural networks trained for classification tasks. We assume that the neural network $\Phi(x; \theta)$ is defined as a sequence of transformations $h_l$ for each of its $L$ layers. The output $ h_L( h_{L-1}( \ldots h_1(x)  ) ) = z_L \in \R^N$ has $N$ logits corresponding to $N$ classes.

In \ourtraining{} we assume that weights of each transformation $h_k$ are located arbitrarily  in the hyperrectangle $[\barbelow W, \bar W]$. For a particular task $T_j$, the Cartesian product of all intervals $[\barbelow W_1, \bar W_1] \times [\barbelow W_2, \bar W_2] \times \ldots [\barbelow W_L, \bar W_L] $ forms the hyperrectangle $\Theta_j$. With this setting we aim to find an upper bound on $\max_{\theta \in \Theta} \ell(T_j, \theta)$. 

Observe that for a given input $x$ and particular interval layer $[\barbelow W, \bar W]$ we can bound the possible outputs of that layer. As such, we can bound not only the weights of our model but also the activations. Thus, we will use interval arithmetic as our main tool to formalize this problem.

Interval arithmetic  \cite{dahlquist2008numerical}(Chapter 2.5.3) is based on the operations on segments.
Let us assume $A$ and $B$ are numbers expressed as intervals. For all $\bar a, \barbelow a, \bar b, \barbelow b \in \R$ where 
$A = [\barbelow a, \bar a]$ , $B = [\barbelow b, \bar b]$, we can define operations such as \cite{lee2004first}:
\begin{itemize}
    \item addition:
    $
     [\barbelow a, \bar a] +  [\barbelow b, \bar b] = [\barbelow a + \barbelow b, \bar a + \bar b]
    $
    \item multiplication:
    $
     [\barbelow a, \bar a] *  [\barbelow b, \bar b] = [\min( \barbelow a * \barbelow b, \barbelow a * \bar b, \bar a * \barbelow b, \bar a * \bar b),
     \max( \barbelow a * \barbelow b, \barbelow a * \bar b, \bar a * \barbelow b, \bar a * \bar b )
     ]
    $      

\end{itemize}
Thus, we can use interval arithmetic to perform affine transformations, which in turn allows us to implement fully-connected and convolutional layers of neural networks. 
Let us consider the interval version of the classical dense layer:
$$
h(x) = Wx + b.
$$
In \ourtraining{} a vector of weights is a vector of intervals and consequently the output of the dense layer is also a vector of intervals:
$$
[ \barbelow h(x), \bar h(x)] = [  \barbelow W , \bar W ] [x, x ].
$$

\our{}  is defined by a sequence of transformations $h_l$ for
each of its $L$ layers. That is, for an input $z_0 = [x,x]$, we have
$$
[\barbelow z_l, \bar z_l] = h_l([\barbelow z_{l-1}, \bar z_{l-1}]) \mbox{ for } l = 1, \ldots, L. 
$$
The output $[\barbelow z_L, \bar z_L] $ has $N$ interval logits corresponding to $N$
classes.

Propagating bounds through any element-wise monotonic activation function (e.g., ReLU, tanh, sigmoid) is trivial. Concretely, if $h_l$ is an element-wise non-decreasing function, we have:
$$
\barbelow z_{l} = h_{l}( \barbelow z_{l-1}),  \qquad
\bar z_{l} = h_{l}( \bar z_{l-1}).
$$
As such, we can use interval arithmetic to push any given input $x$ through the whole network and obtain outputs of $N$ logits intervals corresponding to $N$ classes $[\barbelow z_L, \bar z_L]$.  In other words, we can find $z_L(x; \Theta_j) = [\barbelow z_L, \bar z_L]$ such that $\phi(x; \theta) \in z_L(x; \Theta_j)$ for any $\theta \in \Theta_j$.

\paragraph{Upper bound of maximum loss over a region}
Thanks to the interval arithmetic, we can now find an  upper-bound of $\max_{\theta \in \Theta} \ell(T_j, \theta)$ which is efficient to compute. 

\begin{theorem}
Let $\Phi$ be a neural network, $\Theta$ be a region of parameter space of that neural network, and $z_L(x, \Theta)$ be a function returning a hyperrectangle such that $\Phi(x, \theta) \in z_L(x, \Theta)$ for any $\theta \in \Theta$. Define worst-case interval-based loss as:
$
\hat{l}(x, y; \Theta) =  l(\hat{z}_L, y), 
$
where $\hat{z}_L$ is a vector with each element defined as:
$$
\hat{z}_L^{(i)} = \left\{
\begin{array}{ll}
     \overline{z}_L^{(i)}, & \mbox{ for } y \neq y,\\[0.8ex]
     \underline{z}_L^{(y)}, & \mbox{ otherwise }.
\end{array}
\right.
$$

Then $\hat{l}(x, y; \Theta) \geq \max_{\theta \in \Theta} l(x, \theta)$.
\label{theorem:single_data}
\end{theorem}

\begin{proof}

First, we will show that that $\hat{z}_L$ is the maximum of the cross-entropy loss, i.e. $ \hat{z}_l = \argmax_{z \in z_L(x, \Theta)} l(z, y) $. 
From the definition of $\hat{l}$, for any $\tilde{z} \in z_L(x, \Theta)$, for the worst-case prediction of the true class (when $i = y$) we have 
$
\begin{array}{c}
\tilde{z}^{(y)} \geq \hat z_L^{(y)}.
\end{array}
$
Similarly, for the incorrect class ($i \neq y$)
we have 
$
\begin{array}{c}
\tilde{z}^{(i)} \leq \hat z_L^{i}.
\end{array}
$ 
When we consider coordinate connected with correct label $y=i$ cross-entropy will be larger than any other elements from interval since we return a lower prediction on the correct class. Analogically, when we consider coordinate connected with an incorrect label $i \neq y $ any element larger than the minimum of the interval, we obtain lower cross-entropy since we increase the probability of false classes. From that, it is evident that $\hat{z}_L$ is the maximum argument in this hyperrectangle.

Now, denote the set of possible logits produced by the neural network over the region $\Theta$ as $\Phi(x) = \{\phi(x, \theta) \mid \theta \in \Theta\}$. Since we know that $\phi(x, \theta) \in z_L(x, \Theta)$, we see that $\Phi(x) \subseteq z_L(x, \Theta)$. Since $\hat{z}$ is the maximum over the set $z_L(x, \Theta)$, then it must also be the maximum over its subset $\Phi(x)$.
\end{proof}

This gives us the worst-case scenario over a parameter region for a particular example. It is now straightforward to define $\hat{\ell}(T_j, \Theta) = \frac{1}{|T_j|} \sum_{(x, y) \in T_j} \hat{l}(x, y; \Theta)$ to be the worst-case scenario for the whole task. 
Now we can prove that this $\hat{\ell}$ is an upper bound of the whole continual learning objective.

\begin{theorem}
Assume a sequential learning problem described by \eqref{eq:region_optimization_continual}. Then during task $j$ for any $k < j$, we have $\hat{\ell}(T_{k}, \Theta_j) \geq \max_{\theta' \in \Theta_j} \ell(T_k, \theta') \geq \max_{\theta' \in \Theta_k} \ell(T_k, \theta') $ 
\label{theorem:whole_task}
\end{theorem}

\begin{proof}
 The first inequality follows from Theorem \ref{theorem:single_data}:

\begin{equation}
\begin{split}
\hat{\ell}(T_j, \Theta)) &= \frac{1}{|T_j|} \sum_{(x, y) \in T_j} \hat{l}(x, y; \Theta) \\
                         & \geq \frac{1}{|T_j|} \sum_{(x, y) \in T_j} \max_{\theta \in \Theta} l(x, y; \theta) \\
                         & \geq \max_{\theta \in \Theta} \frac{1}{|T_j|}  \sum_{(x, y) \in T_j}  l(x, y; \theta) \\
                         & = \max_{\theta \in \Theta} \ell(T_j, \theta)
\end{split}
\end{equation}

The second inequality follows directly from the fact that $\Theta_j \subseteq \Theta_k$.
\end{proof}

Analogously to $\hat{\ell}$, we can define $\mathrm{\hat{acc}}(T_j, \Theta_j)$ to be the interval lower bound on minimum accuracy for task $T_j$ over a region of parameter space $\Theta_j$, i.e. $\mathrm{\hat{acc}}(T_j, \Theta_j) \leq \min_{\theta \in \Theta_j} \mathrm{acc}(T_j, \theta)$. Since for each $k > j$, $\Theta_{k} \subseteq \Theta_j$, we can guarantee that throughout the training the accuracy for $j$-th task will not fall below $\mathrm{\hat{acc}}(T_j, \Theta_j)$.

To summarize, we have shown that by using \ourtraining{} we obtain guarantees on non-forgetting since the interval from task $j$ is inside segments from the previous $j - 1$-th task. Moreover, we can solve such a problem in polynomial time, since we restrict our consideration to hyperrectangles. Finally, in Theorem \ref{theorem:whole_task} we demonstrate that the bounds on logits provided by interval arithmetic on neural networks can be used to calculate an efficient upper bound on the whole continual learning objective.

The result presented here is relevant to \citet{DBLP:conf/icml/KnoblauchHD20} who have shown that classical continual learning problem is in general NP-hard because of the highly irregular shapes of the viable regions of parameters. In Section~\ref{sec:theorem} we present a more rigorous treatment of the above derivation, using the setting and assumptions proposed in \citet{DBLP:conf/icml/KnoblauchHD20}.

\section{ \our{} }\label{sec:training}

\begin{algorithm}[t]
\caption{\our{} training procedure for a given task} \label{Alg:train}
\begin{algorithmic}
\STATE Input: model trained on the previous task (weights $W^*$, radii $\e^*$), current task $T_j$. 
\STATE Reparameterize $W, \e$ using Eq. \eqref{eq:interval_reparam}.
\FOR{epoch in $1 \ldots \mathrm{center\_epochs}$} 
    \STATE Update $\mu$ by minimizing $\ell(T_j, \Theta)$    \hfill \COMMENT{Train centers}
\ENDFOR
\STATE Initialize $\e$ as largest possible within the previous interval.
\FOR{epoch in $1 \ldots \mathrm{radii\_epochs}$} 
    \STATE Update $\nu$  by minimizing $\hat{\ell}(T_j, \Theta)$ \hfill \COMMENT{Train radii}
    \IF{$\hat{\mathrm{acc}} \geq \mathrm{acc} \cdot \mathrm{acc\_thresh}$}
        \STATE break
    \ENDIF
\ENDFOR
\STATE \textbf{return} $W$, $\e$
\end{algorithmic}
\end{algorithm}

In the previous section, we have shown that by considering the intervals of activations in a neural network we can derive an upper bound of the maximum loss over a hyperrectangle. Now, we will present a model which minimizes this upper bound in order to approximate the optimization presented in Eq. \eqref{eq:region_optimization_continual}.

\paragraph{Parametrization}

We start by taking a standard feed-forward neural network $\Phi$ with parameters $\theta$, but instead of considering points in the parameter space, we consider regions $\Theta$.
In particular, we describe the $k$-th parameter of the network with the center $W_k$ and the interval radius $\e_k$. Thus, we consider the region:
$$
[\barbelow W_k, \bar W_k] = [W_k - \e_k, W_k + \e_k],
$$
With this formulation, we can still use this network as a standard non-interval model by only using the center weights $W_k$, which will, in turn, produce only the center activations for each layer.
At the same time, we can also keep track of the interval $[\barbelow z_l, \bar z_l]$ by using the interval arithmetic. These operations are independent of each other, so as explained in \citet{gowal2018effectiveness} we can implement them efficiently for GPUs by calculating the center prediction and the intervals in parallel. In practice, we reimplement the basic blocks of neural networks (fully-connected layers, convolutional layers, activations, pooling, etc.) in order to handle interval inputs and interval weights.

With this formulation, we are able to implement most operations in contemporary standard neural network architectures. However, one important exception is batch normalization which uses running statistics of activations in inference mode. Since it is impossible to foresee how these statistics will change in the future, providing reliable forgetting guarantees with interval arithmetic is infeasible. As such, we restrict ourselves to simpler architectures without batch normalization.

During the training of the second and every subsequent task, we need to make sure that the intervals for the current task are contained within intervals from the previous task. Starting from the second task, we keep the previous interval centers $W_k^*$ and radii $\e_k^*$. Then, we reparameterize the training as:
\begin{equation}
\begin{split}
W_k &= W_k^* + \tanh(\mu_k) \cdot \e_k^*  \\
\e_k &= \sigma(\nu_k) \cdot
        \!\begin{aligned}[t]
           \min( & (  W_k^* + \e_k^*) - W_k, \\
           & W_k - (W_k^* - \e_k^*)).
         \end{aligned} 
\label{eq:interval_reparam}
\end{split}
\end{equation}
Observe that we cannot simply set the radius to be $\sigma(\nu_k) \cdot \e_k^*$ as depending on the position on the center $W_k$ the interval $[W_k - \e_k, W_k + \e_k]$ might not be fully contained inside the previous task interval $[W^*_k - \e^*_k, W^*_k + \e^*_k]$. We train $\mu_k$ and $\nu_k$ instead of directly optimizing $W_k$ and $\e_k$.

\paragraph{Training}

When training \our{} in a continual learning setting, we divide training of each task $T_j$ into two phases. In the first phase, we focus on finding the interval centers $W_k$ and keep the radii $\e_k$ frozen. We do this by simply minimizing the cross-entropy loss, as done in standard classification neural network training. In the second phase, we freeze the centers $W_k$ and initialize $\e_k$ with the maximal value that still fits within the previous interval\footnote{For the first task the maximum radius is a hyperparameter to be set.}. Then, we minimize the upper-bound loss $\hat{\ell}(T_j, \Theta_j)$ by optimizing $\e_k$.
Optimizing this upper bound indefinitely would lead to minimizing $\e_k$ to zeroes and producing degenerated intervals. To prevent this, we introduce an interpretable hyperparameter $\mathrm{acc\_thresh}$ which allows us to choose the fraction of examples in the current task that is guaranteed to be classified correctly for the rest of the training. For that purpose, we use the worst-case accuracy $\hat{\mathrm{acc}}$ described in Section~\ref{sec:theoretical} and minimize the worst-case loss $\hat{l}$ until the condition $\hat{\mathrm{acc}} \geq \mathrm{acc} \cdot \mathrm{acc\_thresh}$ is satisfied, with $\mathrm{acc}$ being the train accuracy on the current task. We use running averages on the last few batches to efficiently estimate $\hat{\mathrm{acc}}$ and $\mathrm{acc}$. As such, $\mathrm{acc\_thresh}$ is the variable that controls the stability-plasticity trade-off in our model. A high value of $\mathrm{acc\_thresh}$ will guarantee less forgetting in the end, but will also result in a smaller parameter region $\Theta_j$, which will be used to constrain the parameter search in the next task.
Algorithm \ref{Alg:train} shows the whole training procedure for a single task. 

\paragraph{Analysis}
We believe that realistic computational and memory constraints are an important part of the continual learning setting. In terms of the memory requirements, \our{} needs to remember the centers $W_K^*$ and radii $\e_k^*$ from the previous task. Thus, assuming a network with $D$ parameters, we need to keep additional $2 \cdot D$ values in memory. This is the same memory constraint as for most of our baselines, e.g. online EWC, which remembers the weights and their importance for the last task.

In terms of computational constraints, propagating both activation centers and intervals through the network requires approximately $3$ times as many FLOPs as propagating just a single point. However, the main part of the training is the center optimization phase, that has no overhead as compared to the base network since we only need to push forward a single point (the centers). The interval propagation is only needed in the radii optimization phase that is a small part of the overall training. Thus, in practice it is on a similar order of complexity as, e.g. computing the Fisher information matrix in EWC. Finally, in practice even in the interval propagation phase the training time does not change significantly, as these operations are automatically parallelized on GPUs when using frameworks such as PyTorch.

\section{Experiments}
To verify the empirical usefulness of our method, we test it in three standard continual learning scenarios~\cite{DBLP:journals/corr/abs-1810-12488,vandeVen2019ThreeSF}: incremental task, incremental domain, and incremental class. In the incremental task setting, we create a separate output head for each task and only train a single head per task along with the base network. During inference, the head appropriate for the given task is selected. The incremental domain and incremental class scenarios are more challenging, as they only use a single head and the task identity is not provided during inference. In our case, this single head is created with either 2 (incremental domain) or 10 output classes (incremental class).

We run the experiments on four datasets commonly used for continual learning: MNIST, FashionMNIST, CIFAR-10 and CIFAR-100 split into sequences of tasks. For MNIST and FashionMNIST, we use a standard MLP architecture with 2 hidden layers of size $400$ as previously evaluated in~\cite{DBLP:journals/corr/abs-1810-12488}. For CIFAR-10 and CIFAR-100 we use the fairly simple AlexNet architecture~\cite{krizhevsky2012imagenet}. This choice is dictated by the infeasibility of implementing interval version of batch normalization which is present in most standard convolutional network architectures. In our experiments, we use the Avalanche \cite{lomonaco2021avalanche} library to facilitate easier and fairer comparisons. The code is available at \url{https://github.com/gmum/InterContiNet}.

We compare \our{} with a number of regularization-based approaches typically used in continual learning settings, i.e. EWC, Synaptic Intelligence, MAS, and LwF. As additional baselines, we present vanilla sequential training with SGD and Adam, and a simple $L_2$ regularization between new and old model parameters (i.e. all parameters have the same importance). Further training details are available in Appendix~\ref{sec:appendix-training-details}.

\subsection{MNIST and CIFAR Benchmarks}

\begin{table}[t]
\caption{The average accuracy across all five tasks of the split MNIST protocol, evaluated after learning the whole sequence. Each value is the average of five runs (with standard deviations).}
\label{tab:results_mnist}
\begin{center}
\begin{small}
\setlength{\tabcolsep}{0.15cm}
\begin{tabularx}{\linewidth}{p{2cm}x{1.75cm}x{1.75cm}x{1.75cm}}
\toprule
  Method & Incremental &  Incremental & Incremental \\
     & task &  domain &  class \\
\midrule
SGD         & 96.27 $\pm$ 0.38          & 64.58 $\pm$ 0.26          & 19.01 $\pm$ 0.04 \\
Adam        & 95.53 $\pm$ 3.16          & 59.32 $\pm$ 1.08          & 19.74 $\pm$ 0.01 \\
L2          & 96.31 $\pm$ 0.41          & 72.21 $\pm$ 0.17          & 18.88 $\pm$ 0.18 \\  
\midrule
EWC            & 97.01 $\pm$ 0.13          & 76.90 $\pm$ 0.41          & 18.90 $\pm$ 0.06    \\ 
oEWC      & 97.01 $\pm$ 0.13          & 77.02 $\pm$ 0.51          & 18.89 $\pm$ 0.07    \\
SI          & 96.19 $\pm$ 0.63                      & 80.62 $\pm$ 0.17                        & 17.94 $\pm$ 0.57    \\
MAS         & 96.52 $\pm$ 0.14                      & \bf 84.41 $\pm$ 0.42                        & 17.38 $\pm$ 4.19    \\ 
LwF             & 97.03 $\pm$ 0.05          & 82.76 $\pm$ 0.17          & \bf 49.37 $\pm$ 0.68    \\  
\bf \our{} & \bf 98.93 $\pm$ 0.05 & 77.77 $\pm$ 1.24 & 40.73 $\pm$ 3.26 \\  
\midrule
Offline
                    & 99.74 $\pm$ 0.03 & 99.03 $\pm$ 0.04 & 98.49 $\pm$ 0.02  \\
\bottomrule
\end{tabularx}
\vspace{-1.5em}
\end{small}
\end{center}
\end{table}

\begin{table}[t]
\caption{The average accuracy across all five tasks of the split FashionMNIST protocol, evaluated after learning the whole sequence. Each value is the average of five runs (with standard deviations).}
\label{tab:results_fashionmnist}
\begin{center}
\begin{small}
\setlength{\tabcolsep}{0.15cm}
\begin{tabularx}{\linewidth}{p{2cm}x{1.75cm}x{1.75cm}x{1.75cm}}
\toprule
  Method & Incremental &  Incremental & Incremental \\
     & task &  domain &  class \\
\midrule
SGD          & 92.22 $\pm$ 3.06          & 82.77 $\pm$ 0.44          & 19.91 $\pm$ 0.01 \\
Adam        & 88.13 $\pm$ 5.64          & 78.48 $\pm$ 0.47          & 19.96 $\pm$ 0.01 \\
L2          & 97.36 $\pm$ 0.17          & \bf 92.65 $\pm$ 0.09          & 26.91 $\pm$ 1.23 \\  
\midrule
EWC            & 97.53 $\pm$ 0.15          & 92.12 $\pm$ 0.18          & 19.90 $\pm$ 0.01    \\ 
oEWC      & 96.70 $\pm$ 0.57          & 88.83 $\pm$ 0.30          & 19.87 $\pm$ 0.02    \\
SI        & 97.00 $\pm$ 0.25                      & 91.45 $\pm$ 0.07                        & 19.97 $\pm$ 0.34    \\
MAS         & 97.43 $\pm$ 0.14                     & 91.74 $\pm$ 0.19                    & 10.00 $\pm$ 0.00    \\ 
LwF             & 98.10 $\pm$ 0.07          & 88.63 $\pm$ 0.12          & \bf 39.51 $\pm$ 1.45    \\  
\bf \our{} & \bf 98.37 $\pm$ 0.06 & \bf 92.65 $\pm$ 0.40 & 35.11 $\pm$ 0.02 \\  
\midrule
Offline
                    & 97.98 $\pm$ 0.05 & 96.39 $\pm$ 0.06 & 82.54 $\pm$ 0.13  \\
\bottomrule
\end{tabularx}
\end{small}
\end{center}
\vspace{-1em}
\end{table}

\begin{table}[t]
\caption{The average accuracy across all five tasks of Split CIFAR-10, evaluated after training AlexNet on the whole sequence. Each value is the average of five runs (with standard deviations).}
\label{tab:results_cifar10}
\begin{center}
\begin{small}
\setlength{\tabcolsep}{0.15cm}
\begin{tabularx}{\linewidth}{p{2cm}x{1.75cm}x{1.75cm}x{1.75cm}}
\toprule
  Method & Incremental &  Incremental & Incremental \\
     & task &  domain &  class \\
\midrule
SGD              & 64.74 $\pm$ 8.53          & 69.22 $\pm$ 0.55          & 15.56 $\pm$ 5.07 \\
\midrule
EWC                      &  67.33 $\pm$ 7.11          & 70.26 $\pm$ 0.66          & 11.83 $\pm$ 4.10    \\
oEWC               & 64.59 $\pm$ 7.97          & 65.97 $\pm$ 8.97          & 15.49 $\pm$ 5.01    \\
SI               & 67.26 $\pm$ 7.77                      & 69.69 $\pm$ 0.73                        & 17.38 $\pm$ 4.13    \\
MAS              & 66.20 $\pm$ 9.29                      &  72.54 $\pm$ 0.59                        & 11.79 $\pm$ 4.01    \\ 
 
LwF & {\bf 93.03} $\pm$ 0.28 & {\bf 76.97} $\pm$ 0.91 & 13.89 $\pm$ 5.33 \\  
\bf \our{} & 72.64 $\pm$ 1.18 & 69.48 $\pm$ 1.36 & {\bf 19.07} $\pm$ 0.15 \\  

\midrule
Offline & 93.11 $\pm$ 0.18 & 90.54 $\pm$ 0.32 & 82.72 $\pm$ 0.09  \\
\bottomrule
\end{tabularx}
\end{small}
\end{center}
\vspace{-1em}
\end{table}

Results presented in Tables~\ref{tab:results_mnist} and~\ref{tab:results_fashionmnist} indicate that in the incremental task setup \our{} has better average accuracy across all five tasks than all the other continual learning baselines. We hypothesize that a multi-head setup facilitates our 2-phase training procedure, as the output heads in the last layer are independent and do not have to be constrained inside intervals.
In the incremental domain setup, our method has a similar performance to EWC. However, the advantage of \our{} is that we can guarantee a maximal level of forgetting on the past tasks when training on new data.

The most significant difference is visible in the incremental class scenario, where LwF and \our{} perform much better than typical regularization-based approaches. We can partly explain this behavior with the particular setup, where the final number of output classes is known at the onset of the experiment. While such information induces an indirect regularization effect on the softmax layer in LwF and \our{}, it has smaller effect on other techniques. This dichotomy highlights the fundamental difference between LwF's knowledge distillation approach and all the other baselines, which are prior-focused.

Finally, we extend our investigation to CIFAR-10 with the AlexNet. We choose this architecture as an example of a convolutional network without batch normalization, which is problematic in our setting, as we explain later in Section \ref{sec:limitations}. As shown in Table \ref{tab:results_cifar10}, in this setting \our{} still performs comparably to other parameter regularization methods. At the same time, we observe that LwF tends to significantly outperform other methods, especially in the incremental task setting. We hypothesize that the regularization through functional penalties (KL divergence between outputs) rather than direct parameter regularization works much better for this particular problem \cite{DBLP:conf/iccvw/OrenW21}. In the end, we show that \our{} is able to perform similarly to other methods in the same family while providing guarantees.

\subsection{Analysis of the Parameter Region Size}

Although the stability-plasticity dilemma is a crucial problem in all CL methods, in \our{} its impact can be directly observed by investigating the region of valid parameters $\Theta$. The bigger the region, the higher the plasticity, as we are able to use more parameter combinations and thus represent more functions. At the same time, in order to achieve stability and get guarantees on forgetting, we need to restrict the size of this region. Here, we investigate the question of whether the interval training will at some point lead to an area of zero volume, resulting in a state without any plasticity and effectively stopping the learning\footnote{We note that this risk also occurs for other CL methods (e.g. accumulating high EWC penalties which prohibit learning or reducing the percentage of trainable parameters in PackNet to zero).}.

\begin{figure}[t]
    \centering
    \includegraphics[width=0.9\linewidth]{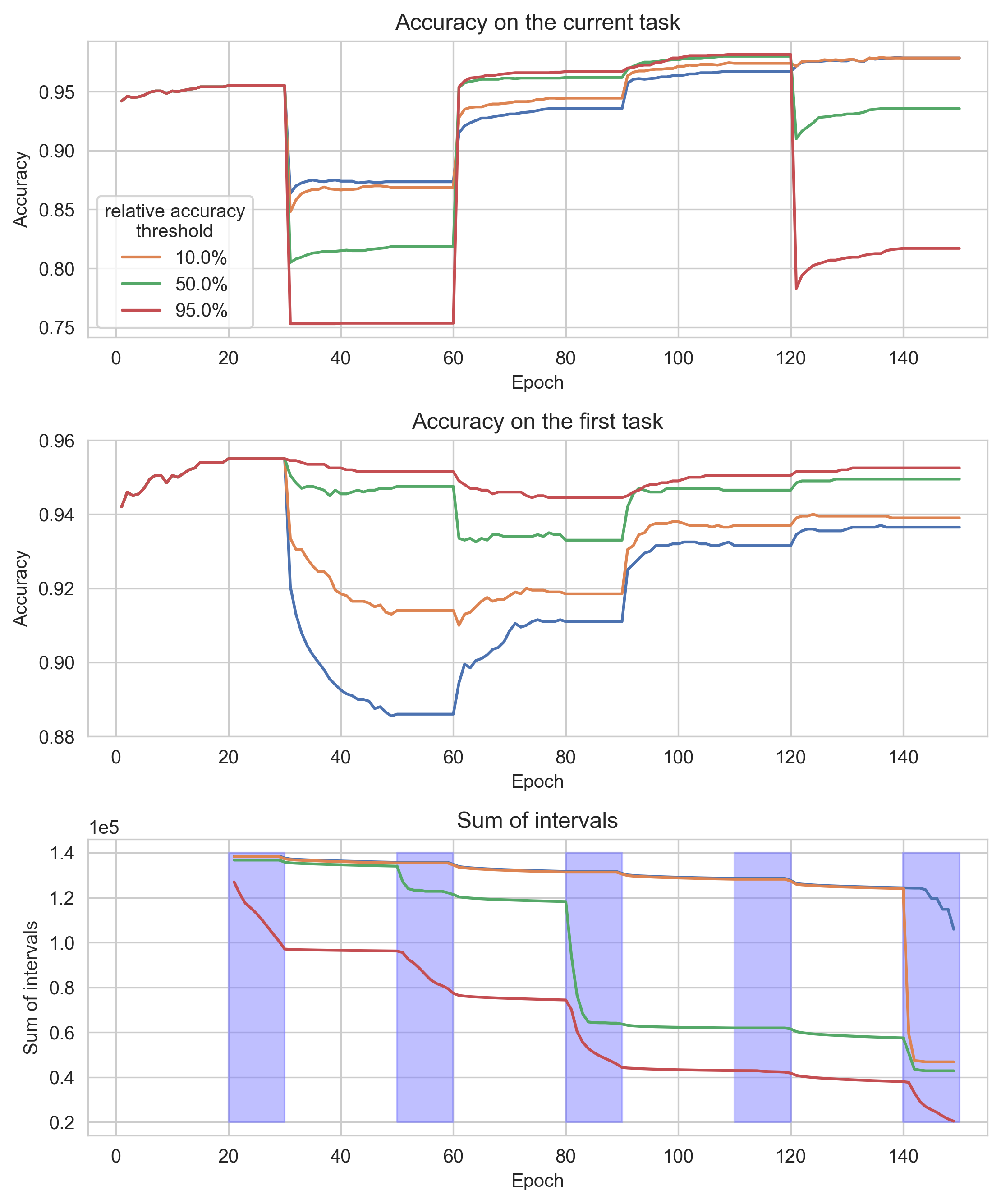}
    \caption{\our{} trained on split-Fashion MNIST in the incremental domain setting, with different values of the $\mathrm{acc\_thresh}$ hyperparameter. We show accuracy on the current task (top, representing plasticity) and the first task (middle, representing stability)  to show how the hyperparameter choices affect the final results. Additionally, the bottom plot shows how the valid parameter region changes throughout the training -- highlighted areas denote the radii optimization phase, which shrinks this region most effectively.}
    \label{fig:thresholds}
    \vspace{-1em}
\end{figure}

In practice, we control the pace at which the volume decreases with hyperparameters, e.g. by choosing the worst-case accuracy threshold -- if we set this threshold high, the region collapses quicker. Here, we conduct additional studies to better understand how the threshold hyperparameter impacts the results and the dynamics of the region of valid parameters. In Figure~\ref{fig:thresholds} we present the accuracy curves for training \our{} on FashionMNIST in the incremental domain setting, using different values of the accuracy threshold hyperparameter. This threshold controls the percentage of the correctly classified examples that we guarantee will not be forgotten throughout the training. For example, in Figure~\ref{fig:thresholds} the method obtains $95.3\%$ accuracy at the end of the first task. If we set $\mathrm{acc\_thresh} = 95\%$, then we can guarantee that at no point during training the accuracy on the first task will be lower than $95\% \cdot 95.3\% \approx 90.5\%$. 

We see that this hyperparameter indeed controls the plasticity-stability trade-off, with larger values contributing to less forgetting, but also less plasticity on the subsequent tasks. Note that in practice the obtained accuracy is much higher than the one guaranteed by the worst-case accuracy. For example, even with $\mathrm{acc\_thresh} = 0.05$ we are able to maintain first-task accuracy above $90\%$ through the whole training. However, this performance is not guaranteed and might deteriorate with further training.

In the bottom part of Figure~\ref{fig:thresholds}, we plot how the valid parameter region changes throughout the training. Since the volume can get easily degenerated (e.g. if we prohibit even one parameter from changing), we summarize it by reporting the sum of all radii. Confirming the intuition, the valid region for models with higher $\mathrm{acc\_thresh}$ shrinks much faster, while for lower thresholds the impact is not as noticeable. In the end, we observe that even in the most restrictive setting, there is still room for learning.

\begin{figure}[t]
    \centering
    \includegraphics[width=0.96\linewidth]{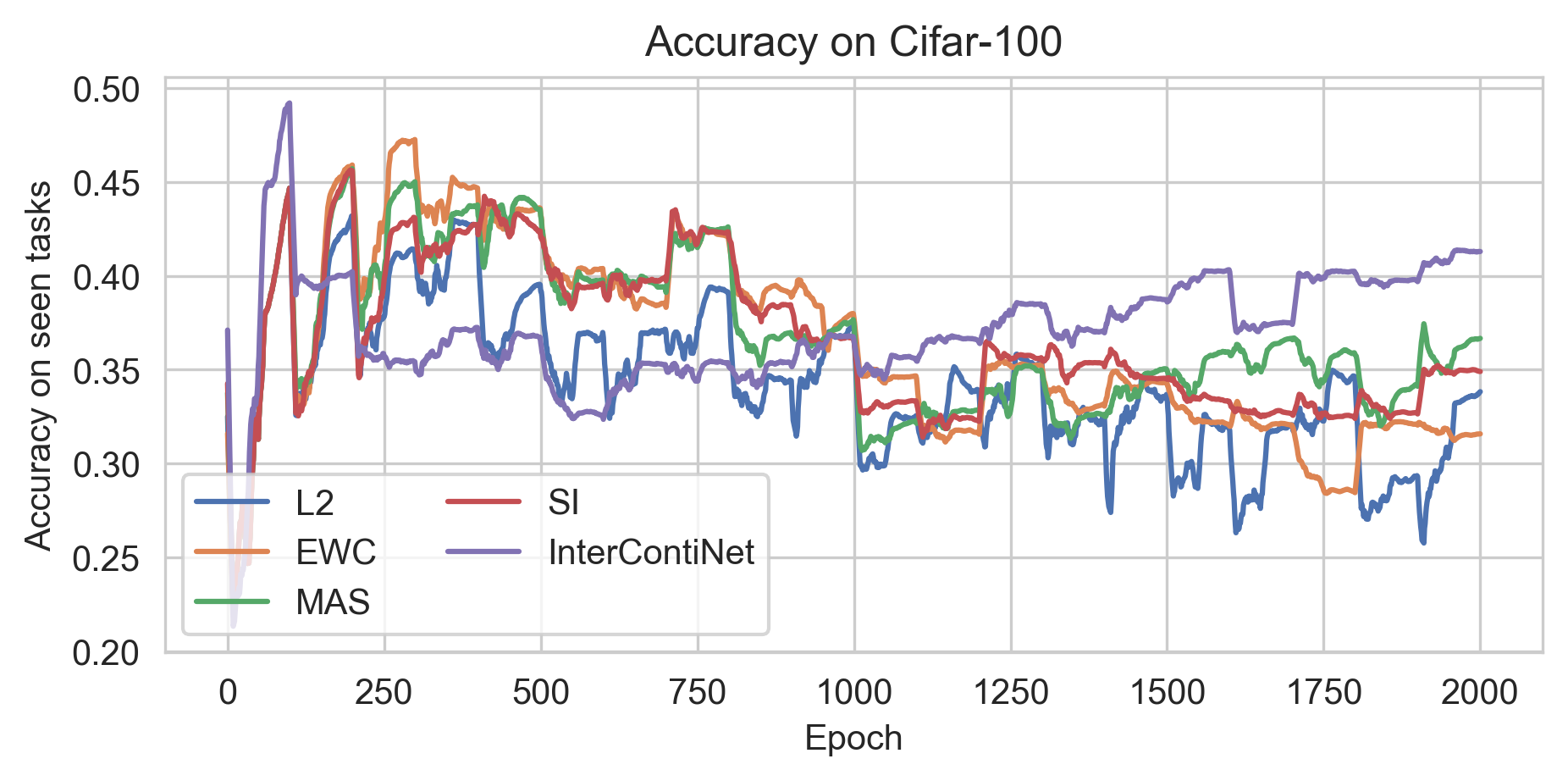}
    \caption{Results for parameter regularization methods for split-CIFAR-100 with 20 tasks, averaged over 3 seeds. \our{} is still able to learn new tasks at the end of the training.}
    \label{fig:results_cifar100}
\end{figure}

Finally, we extend the experiments on the collapse of the valid parameter region to a setting with a higher number of tasks. We use split-CIFAR-100 with $20$ tasks of $5$ classes and focus on the incremental task scenario. The results in Figure~\ref{fig:results_cifar100} show that \our{} performs on similar scale as other parameter regularization methods and is still able to learn new tasks at the end of the training. We consider this as additional evidence that interval constraints do not necessarily lead to faster plasticity collapse than soft regularization.
Additionally, we ran LwF in this setting and found out that it significantly outperforms the parameter regularization methods by a very wide margin. We show the additional results in Figure~\ref{fig:results_cifar100_full} in Appendix~\ref{sec:appendix-training-details}.
This finding is consistent with the split-CIFAR-10 results and recent analysis~\cite{DBLP:conf/iccvw/OrenW21}, and is a strong motivation to consider functional regularization in future work.

\section{Limitations \& Future Work}
\label{sec:limitations}

Although reformulating continual learning as a sequential contraction of the model's parameter space creates a viable theoretical framework, \our{} has several limitations that would require further research to be addressed.

To begin with, interval arithmetic cannot be easily combined with batch normalization layers. As we cannot foresee the running statistics of activations at future inference, we cannot guarantee the bounds of the model's output. This limitation is problematic due to the ubiquitousness of batch normalization in contemporary neural network models. We plan to evaluate \our{} with different normalization schemes (e.g. layer \cite{DBLP:journals/corr/BaKH16} or group normalization \cite{DBLP:journals/ijcv/WuH20}), where providing guarantees is still possible. Another option worth exploring is freezing the batch normalization statistics after the first task.

On the other hand, the no-forgetting guarantees are also limited in several ways. Firstly, we concentrate on the worst-case scenario, i.e.~we maintain that the worst solution in the allowed parameter space will not get even worse. In practice, we evaluate the model at the interval centers, which can leave a lot of leeway for the accuracy to fluctuate above the defined threshold. Secondly, these guarantees are computed on the training set. Depending on the amount of training--test distribution shift, a non-negligible change in the model's response can be unaccounted for in our worst-case scenario.

Another drawback of hard constraints is the risk of reducing the viable parameter space to a zero-volume region when exposed to very long sequences of tasks or very deep architectures. While it is a very reasonable concern, the goal of our empirical evaluation was to show that \our{}'s parameter space does not degenerate so rapidly in typical experimental setups. Nevertheless, based on Theorem~\ref{theorem:single_data}, it is theoretically possible to devise a malicious dataset that would result in arbitrarily large forgetting or complete loss of plasticity.
Lastly, our evaluation is limited to classification problems only. However, extensions to other types of problems (e.g. regression) could be made with simple adjustments to the loss function.

To summarize, in this work we proposed \ourtraining{}, a general approach to continual learning. We constrain the parameter search within the hyperrectangle representing the region of valid solutions for the previous task. We show that by applying interval arithmetic to operations present in standard neural networks in this setting we can derive a tractable and differentiable upper bound of the CL objective. Then we proposed and empirically evaluated \our{}, a CL algorithm that minimizes this upper bound to set hard limits on forgetting.

\section{Acknowledgements}

The work of M. Wołczyk, K. Piczak, J. Tabor, and T. Trzciński was supported by Foundation for Polish Science (grant no POIR.04.04.00-00-14DE/18-00) carried out within the Team-Net program co-financed by the European Union under the European Regional Development Fund.
The work of P. Spurek was supported by the National Centre of Science (Poland) Grant No. 2019/33/B/ST6/00894. The work of T. Trzciński was supported by the National Centre of Science (Poland) Grant No. 2020/39/B/ST6/01511. The work of B. Wójcik was supported by the Priority Research Area DigiWorld under the program Excellence Initiative - Research University at the Jagiellonian University in Kraków. 

\bibliography{ref_1}
\bibliographystyle{icml2022}

\newpage
\appendix
\onecolumn

\section{Training Details}
\label{sec:appendix-training-details}

\paragraph{Data}
We use MNIST, FashionMNIST, CIFAR-10, and CIFAR-100 datasets with original training and testing splits. We do not use any form of data augmentation or normalization. We split the datasets into 5 tasks containing data from classes: [0, 1], [2, 3], [4, 5], [6, 7], [8, 9]. We split CIFAR-100 into 20 tasks with 5 classes.

\paragraph{Architecture}
The baseline model for MNIST and FashionMNIST is a MLP with 2 hidden layers of 400 units and a linear output layer (single- or multi-head). We use ReLU activations between layers. For CIFAR-10 and CIFAR-100 we use the AlexNet~\cite{krizhevsky2012imagenet} architecture adapted to $32 \times 32$ inputs. For \our{} we simply replace standard fully-connected layers with its interval variants, with the linear layers of the heads in the incremental task setting being an exception. That is, those layers also accept the lower and upper bounds as input, but we do not add $\mu$ and $\nu$ for them, as the weights for these layers are task-specific.

\paragraph{Hyperparameters}
In the MNIST and Fashion-MNIST experiments we use batch size of 128, 30 epochs of training for each task (i.e. 150 in total) and vanilla cross entropy loss. In the offline (upper bound approximation) learning experiment we use a single combined task with 100 epochs of training. We repeat each experiment 5 times with different seeds to get the mean and the standard deviation estimates.

In the SGD baseline and every continual learning baseline we use the SGD optimizer without momentum. The learning rate is set to $0.001$ for  MNIST and FashionMNIST experiments, and $0.01$ for split-CIFAR-10 and split-CIFAR-100. In CIFAR experiments we additionally decrease the learning rate $10$x in the 50th and the 90th epoch. 

Tables \ref{tab:baseline_mnist_hps}, \ref{tab:baseline_c10_hps}, \ref{tab:baseline_c100_hps} contain chosen hyperparameters for MNIST, CIFAR-10 and CIFAR-100 respectively, for the baseline methods.

\begin{table}[h]
\caption{Hyperparameters for the baseline methods used in the MNIST and FashionMNIST experiments.}
\label{tab:baseline_mnist_hps}
\begin{center}
\begin{small}
\setlength{\tabcolsep}{0.15cm}
\begin{tabularx}{0.47\textwidth}{p{3.5cm}x{2.5cm}x{1.5cm}}
\toprule
Method & Hyperparameter & Value \\
\midrule
L2 & $\lambda$ & 0.1 \\
EWC & $\lambda$ & 2048 \\
Synaptic Intelligence & $\lambda$ & 2048 \\
MAS & $\lambda$ & 1 \\
\multirow{2}{=}{LwF} & temperature & 0.5 \\
& $\alpha$ & 0.5 \\
\bottomrule
\end{tabularx}
\end{small}
\end{center}
\end{table}

\begin{table}[h]
\caption{Hyperparameters for the baseline methods used in the CIFAR-10 experiments.}
\label{tab:baseline_c10_hps}
\begin{center}
\begin{small}
\setlength{\tabcolsep}{0.15cm}
\begin{tabularx}{0.47\textwidth}{p{3.5cm}x{2.5cm}x{1.5cm}}
\toprule
Method & Hyperparameter & Value \\
\midrule
L2 & $\lambda$ & 0.05 \\
EWC & $\lambda$ & 0.05 \\
Synaptic Intelligence & $\lambda$ & 1. \\
MAS & $\lambda$ & 0.005 \\
\multirow{2}{=}{LwF} & temperature & 1. \\
& $\alpha$ & 1. \\
\bottomrule
\end{tabularx}
\end{small}
\end{center}
\end{table}

\begin{table}[h]
\caption{Hyperparameters for the baseline methods used in the CIFAR-100 experiments.}
\label{tab:baseline_c100_hps}
\begin{center}
\begin{small}
\setlength{\tabcolsep}{0.15cm}
\begin{tabularx}{0.47\textwidth}{p{3.5cm}x{2.5cm}x{1.5cm}}
\toprule
Method & Hyperparameter & Value \\
\midrule
L2 & $\lambda$ & 0.005 \\
EWC & $\lambda$ & 0.0005 \\
Synaptic Intelligence & $\lambda$ & 0.1 \\
MAS & $\lambda$ & 0.0005 \\
\multirow{2}{=}{LwF} & temperature & 1. \\
& $\alpha$ & 0.5 \\
\bottomrule
\end{tabularx}
\end{small}
\end{center}
\end{table}

Since $\hat{\ell}$ and $\ell$ have values differing by orders of magnitude, we use a separate learning rates for training centers $W$ and radii $\e$. We tune these two learning rates and the $\mathrm{acc\_thresh}$ hyperparameter and list them in Table \ref{tab:interval_hyperparams}. Similarly as in the baselines, we use SGD without momentum. After each task the $\nu$ parameters are reset to the value of $5.0$.

\begin{table}[h]
\caption{Hyperparameters for \our{} on each dataset and setting combination.}
\label{tab:interval_hyperparams}
\begin{center}
\begin{small}
\setlength{\tabcolsep}{0.10cm}
\begin{tabularx}{0.8\textwidth}{p{1.5cm}x{1.25cm}x{1.5cm}x{3.0cm}x{3.0cm}x{1.5cm}}
\toprule
 Dataset & Setting & $\mathrm{acc\_thresh}$ & Center optimization LR & Radii Optimization LR & Initial radius \\
\midrule
\multirow{3}{=}{MNIST} & IT & $0.9$ & $1.0$ & $100$ & 1 \\
 & ID & $0.8$ & $1.0$ & $1000$ & $1$ \\
 & IC & $0.8$ & $0.001$ & $1.0$ & $1$ \\
 \midrule
\multirow{3}{=}{FashionMNIST} & IT & $0.9$ & $0.001$ & $100$ & 1 \\
 & ID & $0.8$ & $0.001$ & $10$ & $1$ \\
 & IC & $0.9$ & $0.001$ & $0.1$ & $1$ \\
 \midrule
\multirow{3}{=}{CIFAR-10} & IT & $0.1$ & $0.1$ & $0.1$ & $0.1$ \\
 & ID & $0.1$ & $0.1$ & $10^{-10}$ & $0.1$ \\
 & IC & $0.1$ & $0.1$ & $10^{-20}$ & $0.1$ \\
 \midrule
\multirow{1}{=}{CIFAR-100} & IT & $0.05$ & $0.001$ & $100$ & $0.5$ \\
\bottomrule
\end{tabularx}
\end{small}
\end{center}
\end{table}

\begin{figure}[t]
    \centering
    \includegraphics[width=0.9\linewidth]{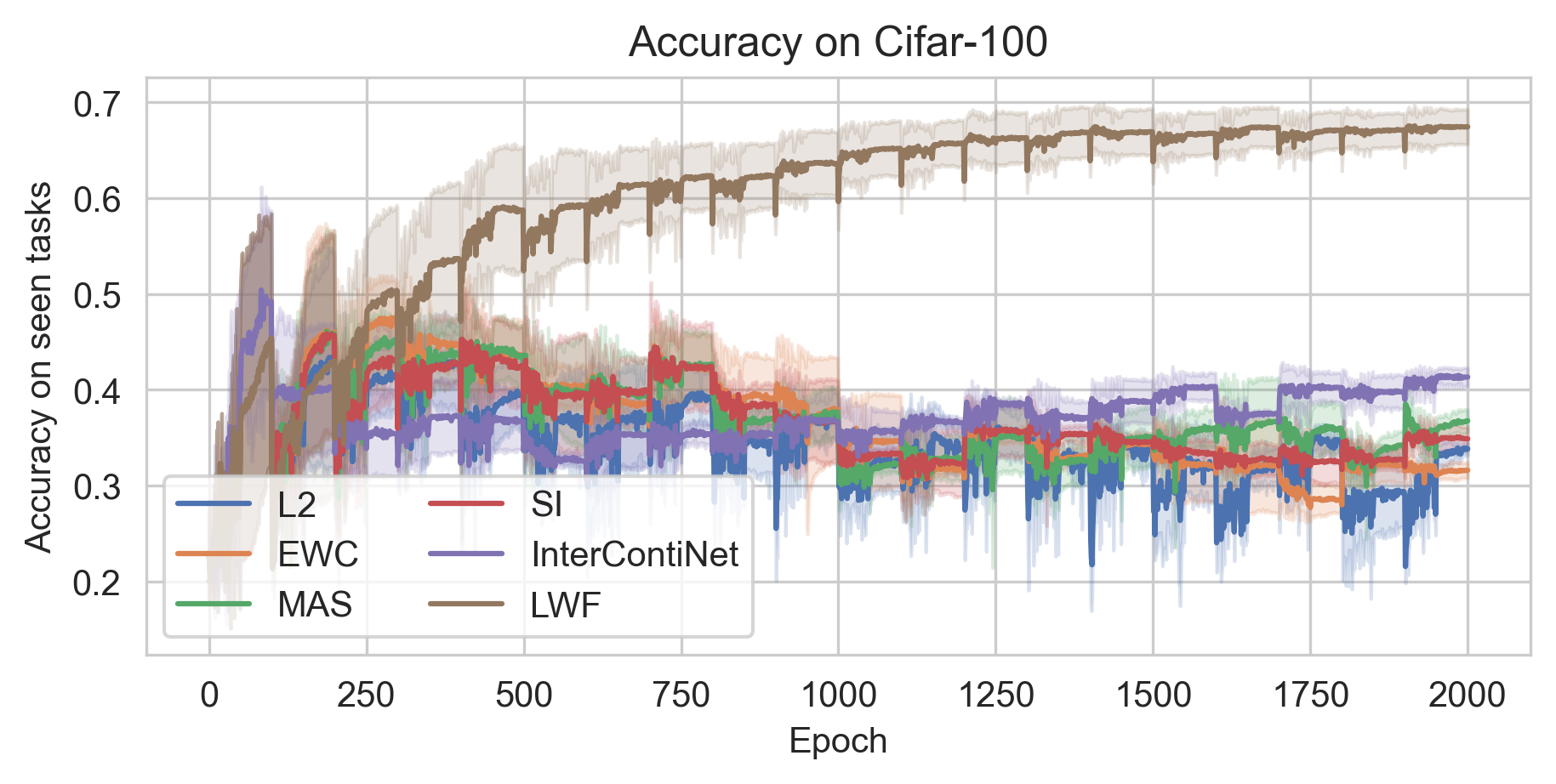}
    \caption{Full results for CIFAR-100, computed from 3 seeds. The confidence intervals represent standard deviation.}
    \label{fig:results_cifar100_full}
\end{figure}

\paragraph{Additional results} In Figure~\ref{fig:results_cifar100_full} we show full results for split CIFAR-100, including LwF and standard deviation confidence regions.

\section{Formal Presentations of the Results}\label{sec:theorem}

In \cite{DBLP:conf/icml/KnoblauchHD20} authors develop a
theoretical approach that explains why catastrophic forgetting is a challenging problem. Furthermore, the authors show that optimal Continual Learning algorithms generally have to solve a NP-HARD problem.
Thanks to interval arithmetic we consider only hyperrectangles. By constraining the CL problem to interval architecture we reduce the complexity to a polynomial one. In the experimental section, we show that such a solution also gives excellent practical results.

Roughly speaking, the main problem with the continual learning setting is that areas in the parameter space that guarantee each task's performance can have arbitrary shape, see Fig.~\ref{fig:teaser}. Thus, our primary goal is to detect weight in the intersection of such regions. 
Even simple linear models together with an intuitively appealing upper bound on the prediction error as optimality criterion are NP-HARD (see \cite{DBLP:conf/icml/KnoblauchHD20}[Example 1]).

In our framework, we constrain the areas in the parameter space which guarantee the solution of each task to hyperrectangles, see Fig.~\ref{fig:imp_dec_1_b} (bottom). Thus, the intersection of hyperrectangles is a hyperrectangle, and we can produce that intersection in a polynomial time. 

For convenience of the reader we introduce the notation from \cite{DBLP:conf/icml/KnoblauchHD20}.  
We deal with random variables $\bf X_t$, $\bf Y_t$. Realizations of the random variable $\bf X_t$ live on the input space $\mathcal{X}$ and provide information about the random outputs $\bf Y_t$
with realizations on the output space $\mathcal{Y}$.
Throughout, $P(A)$ denotes the collection of all probability measures on $A$.

\begin{definition} (Tasks). For a number $T \in \N$ and random
variables ${(X_t,Y_t)}^T_{t=1}$ defined on the same spaces $X$
and $Y$, the random variable $(X_t,Y_t)$ is the $t$-th task, and
its probability space is $(X_t \times Y_t, \Sigma, \P_t)$, where $\Sigma$ is a $\sigma$-algebra and $ \P_t$ a probability measure on $X_t \times Y_t \subset \mathcal{X} \times \mathcal{Y}$.
\end{definition}

Given a sequence of samples from task-specific random
variables, a CL algorithm sequentially learns a predictor for $Y_t$ given $X_t$. This means that there will be some hypothesis
class $F_{\Theta}$ consisting of conditional distributions which allow
(probabilistic) predictions about likely values of $Y_t$.

\begin{definition} (CL hypothesis class). The CL hypothesis
class $F_{\Theta}$ is parameterized by $\Theta$: For any $f \in F_{\Theta}$, there exists a $\theta \in \Theta$ so that $f_{\theta} = f$. More precisely, $F_{\Theta} \subset P(\mathcal{Y})^{\mathcal{X}}$  if the task label is not conditioned on. Alternatively, $F_{\Theta} \subset P(\mathcal{Y})^{X \times \{1,2,...T\}}$
if the label is conditioned on.
\end{definition}

\begin{definition}
 (Continual Learning). For a CL hypothesis
class $F_{\Theta}, T \in \N$ and any sequence of probability measures $\{\hat \P_t\}^T_{t=1}$ such that $\hat \P_t \in \mathcal{P}(\mathcal{Y}_t)^{\mathcal{X}_t} \subseteq P(Y)^{\mathcal{X}}$ , CL
algorithms are specified by functions
$$
\begin{array}{c}
\hat{\mathcal{A}_{I}} : \Theta \times I \times \mathcal{P}(\mathcal{Y})^{\mathcal{X}} \to I \\
\hat{\mathcal{A}_{\Theta}} : \Theta \times I \times P(\mathcal{Y})^{\mathcal{X}} \to \Theta,
\end{array}
$$
where $I$ is some space that may vary between different CL
algorithms. Given $A_I$ and $A_{\theta}$ and some initializations $\theta_0$ and $I_0$, CL defines a procedure given by
$$
\begin{array}{c}
\theta_1 = \hat {\mathcal{A}}_{\Theta}(\theta_0,I_0,\hat{\P_1})\\
I_1 = \hat{\mathcal{A}}_{I}(\theta_1,I_0,\hat{\P}_1)\\
\theta_2 = \hat {\mathcal{A}}_{\Theta}(\theta_1,I_1,\hat{\P}_2)\\
I_2 = \hat{\mathcal{A}}_{I}(\theta_2,I_1,\hat{\P}_2)\\
\ldots \\
\theta_T = \hat {\mathcal{A}}_{\Theta}(\theta_{T-1},I_{T-1},\hat{ \P}_{T})\\
I_T = \hat{\mathcal{A}_{I}}(\theta_T,I_{T-1},\hat{\P}_T) \quad \ \\
\end{array}
$$

\end{definition}

In \cite{DBLP:conf/icml/KnoblauchHD20} authors introduce very flexible  formalism. In
particular, all that it needs is an arbitrary binary-valued optimality criterion $\mathcal{C}$, whose function is to assess whether or not information of a task has been retained ($\mathcal{C} = 1$) or forgotten $(\mathcal{C} = 0)$. According to this formalism, a CL algorithm avoids catastrophic forgetting (as judged by the criterion $\mathcal{C}$) if and only if its output at task $t$ is guaranteed to satisfy $\mathcal{C}$ on all previously seen tasks. In this context, different ideas
about the meaning of catastrophic forgetting would result
in different choices for $\mathcal{C}$. As we will analyze CL with the tools of set theory, it is also convenient to define the function $SAT$, which maps from task distributions into the subsets consisting of all values in $\Theta$ which satisfy the criterion $\mathcal{C}$ on the given task.

\begin{definition} 
For an optimality criterion
$ \mathcal{C}: \Theta \times \mathcal{P}(\mathcal{X} \times \mathcal{Y}) \to \{0,1\} $ and a set $\mathcal{Q} \subseteq  \mathcal{P}(\mathcal{X} \times \mathcal{Y})$  of task distributions,
the function $SAT: \mathcal{P}(\mathcal{X} \times \mathcal{Y}) \to 2^{\Theta}$  defines the subset of $\Theta$ which satisfies $\mathcal{C}$ and is given by
$$
SAT(\hat \P ) = \{\theta \in  \Theta : \mathcal{C}(\theta, \hat \P ) = 1\}.
$$

The collection of all possible sets generated by SAT is
$$
SAT_{\mathcal{Q}} = \{SAT(\hat \P) : \hat \P \in \mathcal{Q}\}
$$
and the collection of finite intersections from SATQ is
$$
SAT_{\cup} = \{ \cup_{i=1}^t A_i : A_i \in SAT_{\mathcal{Q}}, 1 \leq i \leq t \mbox{ and } 1 \leq t \leq T, T \in \N \}.
$$

\end{definition}

Now we are ready to introduce our theorem. We assume that for a task, we can train our model in intervals. In the experimental section, we show that such architecture can be efficiently trained.

\begin{theorem}
Take $F_{\Theta}$ to be the collection of MLP neural networks  $\phi_{\theta}$ 

with inputs on $\mathcal{X}$ and outputs on $\mathcal{Y} \subset \R$ linked through the
coefficient vector $\theta \in \Theta$. Further, let $\mathcal{Q}$ be the collection of empirical measures
$$
\hat m^{t}(y,x) = \frac{1}{n_t} \sum_{i=1}^{n_t} \delta_{(y_i^t,x_i^t)} (y,x)  
$$
whose $n_t \in \N$ atoms $\{ (y_i^t,x_i^t) \}_{i=1}^{n_t}$ represent t-th task.

We assume that for all $i = 1, \ldots, n_t $ and $\e \geq 0$ exist $[\barbelow W_i, \bar W_i]$ such that for all $\theta \in [\barbelow W_i, \bar W_i]$ we have $| y_i^t - \phi_{\theta}(x_i^t) | \leq \e$. 

Further, define for $\e \geq 0$ and all $\hat \P \in \mathcal{Q}$ the criterion
$$
\mathcal{C}(\theta,\hat \P) = \left\{
\begin{array}{ll}
     1, & \mbox{ if } | y_i^t - \phi_{\theta}(x_i^t) | \leq \e \mbox{ and } \theta \in [\barbelow W_i, \bar W_i] \mbox{ for all } i=1 \ldots n_t \\[0.8ex]
     0, & \mbox{ otherwise }.
\end{array}
\right.
$$

Then,
$$
\bigcap_{i = 1}^{n_t} [\barbelow W_i, \bar W_i] \subset SAT(\hat \P),
$$
\end{theorem}

\begin{proof}
Then, it is straightforward to see that
$$
\begin{array}{c}
SAT(\hat \P) = \{ \theta \in \Theta : y_i^t - \phi_{\theta}(x_i^t) \leq \e \mbox{ and }  y_i^t - \phi_{\theta}(x_i^t) \geq - \e \mbox{ and } \theta \in [\barbelow W_i, \bar W_i],  \mbox{ for all } i=1 \ldots n_t \} = \\[0.8ex]
= \left( \cap_{i=1}^{n_t}  \{ \theta \in \Theta : y_i^t - \phi_{\theta}(x_i^t) \leq \e \} \right) \cap \left( \cap_{i=1}^{n_t}  \{ \theta \in \Theta : y_i^t - \phi_{\theta}(x_i^t) \geq -\e \} \right) \cap (\cap_{i = 1}^{n_t} [\barbelow W_i, \bar W_i] ) \\
\supseteq \cap_{i = 1}^{n_t} [\barbelow W_i, \bar W_i] . 
\end{array}
$$
The intersection of hyperrectangles is a hyperrectangle.
\end{proof}

Because the rectangle intersection algorithm can be performed in polynomial time, our solution can be found in polynomial time as well. Furthermore, intersection of hyperrectangles from all tasks $\cap_{i = 1}^{n_t} [\barbelow W_i, \bar W_i]$ is subset $SAT(\hat \P)$, we have guaranteed not to forget previous tasks.


\end{document}